\documentclass[runningheads]{llncs}
\usepackage{graphicx}
\usepackage{comment}
\usepackage{amsmath,amssymb} 
\usepackage{color}

\usepackage{cite}
\usepackage[caption=false]{subfig}
\usepackage{multirow}
\usepackage{booktabs}

\usepackage{amsthm}
\newtheorem{prop}{Proposition}
\usepackage{adjustbox}
\usepackage{hyperref}
\hypersetup{
    colorlinks=true,
    linkcolor=magenta,
    filecolor=magenta,
    urlcolor=magenta,
}

\begin{document}
\pagestyle{headings}
\mainmatter
\def\ECCVSubNumber{3796}  

\title{Shape Prior Deformation for Categorical 6D Object Pose and Size Estimation} 

\titlerunning{Shape Prior Deformation for Categorical 6D Pose Estimation}
%
\author{Meng Tian\orcidID{0000-0001-9937-8975}, Marcelo H Ang Jr\orcidID{0000-0001-8277-6408}, \\ Gim Hee Lee\orcidID{0000-0002-1583-0475}}
%
\authorrunning{Meng Tian, Marcelo H Ang Jr, and Gim Hee Lee}
%
\institute{National University of Singapore, Singapore \\
\email{tianmeng@u.nus.edu, \{mpeangh, gimhee.lee\}@nus.edu.sg}}
\maketitle

\begin{abstract}

We present a novel learning approach to recover the 6D poses and sizes of unseen object instances from an RGB-D image.
To handle the intra-class shape variation, we propose a deep network to reconstruct the 3D object model by explicitly modeling the deformation from a pre-learned categorical shape prior.
Additionally, our network infers the dense correspondences between the depth observation of the object instance and the reconstructed 3D model to jointly estimate the 6D object pose and size.
We design an autoencoder that trains on a collection of object models and compute the mean latent embedding for each category to learn the categorical shape priors.
Extensive experiments on both synthetic and real-world datasets demonstrate that our approach significantly outperforms the state of the art.
Our code is available at \url{https://github.com/mentian/object-deformnet}.


\end{abstract}

\section{Introduction}

Accurate 6D object pose estimation plays an important role in a variety of tasks, such as augmented reality, robotic manipulation, scene understanding, etc.
In recent years, substantial progress has been made for instance-level 6D object pose estimation, where the exact 3D object models for pose estimation are given.
Unfortunately, these methods \cite{peng2019pvnet, zakharov2019dpod, wang2019densefusion} cannot be directly generalized to category-level 6D object pose estimation on new object instances with unknown 3D models. Consequently, the category, 6D pose and size of the objects have to be concurrently estimated. 
Although some other object instances from each category are provided as priors, the high variation of object shapes within the same category
makes generalization to new object instances extremely challenging.

To the best of our knowledge, \cite{sahin2018category} is the first work to address the 6D object pose estimation problem at category-level.
This approach defines 6D pose on semantically selected centers and trains a part-based random forest to recover the pose.
However, building part representations upon 3D skeleton structures limits the generalization capability across unseen object instances.
Another work \cite{wang2019normalized} proposes the first data-driven solution and creates a benchmark dataset for this task.
They introduce the Normalized Object Coordinate Space (NOCS) to represent different object instances within a category in a unified manner.
A region-based network is trained to infer correspondences 
from object pixels to the points in NOCS.
Class label and instance mask of each object are also obtained at the same time.
These predictions are used together with the depth map to estimate the 6D pose and size of the object via point matching. However, the lack of explicit representation of shape variations limits their performance. 

In this work, we propose to reconstruct the complete object models in the NOCS to capture the intra-class shape variation. More specifically, we first learn the categorical shape priors from the given object instances, and then train a network to estimate the deformation field of the shape prior (that is used to get the reconstructed object model) and
the correspondences between object observation and the reconstructed model. 
The shape prior serves as prior knowledge of the category and encodes geometrical characteristics that are shared by objects of a given category.
The deformation predicted by our network captures the instance-specific shape details, i.e. shape variation of that particular instance.
We present a method which is applicable across different object categories and data representations to learn the shape priors.
In particular, an autoencoder is trained on a collection of object models from various categories.
For each category, we compute the mean latent embedding over all instances in the respective category.
The categorical shape prior is constructed by passing the mean embedding through a decoder.
Note that there is no restriction on the data representation (point cloud, mesh, or voxel) of shape priors or collected models as long as we choose a proper architecture for the encoder and decoder.

We use the Umeyama algorithm \cite{umeyama1991least} to recover the 6D pose and metric size of the object from the correspondences estimated by our network that maps 
the point cloud obtained from the observed depth map to the points in NOCS.
We evaluate our method on two standard benchmarks.
Extensive experiments demonstrate the advantage of our network and prove the effectiveness of explicitly modeling the deformation. In summary, the main contributions of this work are:
\begin{itemize}
	\item We propose a novel deep network for category-level 6D object pose and size estimation; our network explicitly models the deformation from the categorical shape prior to the object model.
	\item We present a learning-based method which utilizes the latent embeddings to construct the shape prior; our method is applicable across different categories and data representations.
	\item Our network achieves significantly higher mean average precisions on both synthetic and real-world benchmark datasets.
\end{itemize}

\section{Related Work}

\noindent\textbf{Instance-Level Pose Estimation.}
Existing instance-level pose estimation approaches broadly fall into three categories.
The first category casts votes in the pose space and further refines coarse pose with algorithms such as iterative closest point.
LINEMOD \cite{hinterstoisser2012model} uses holistic template matching to find the nearest viewpoint.
\cite{sundermeyer2018implicit} generates a latent code for the input image and search for its nearest neighbor in the codebook.
\cite{tejani2014latent, kehl2016deep} aggregate the 6D votes cast by locally-sampled RGB-D patches.
The second category directly maps input image to object pose.
\cite{kehl2017ssd, xiang2017posecnn} extend 2D object detection network such that it can predict orientation as an add-on to the identity and 2D bounding box of the object.
\cite{li2018unified, wang2019densefusion} regress 6D pose from RGB-D images in an end-to-end framework.
The third category relies on establishing point correspondences.
\cite{brachmann2014learning, krull2015learning, michel2017global} regress the corresponding 3D object coordinate for each foreground pixel.
\cite{rad2017bb8, tekin2018real, peng2019pvnet} detect the keypoints of the object on image and then solve a Perspective-n-Point problem.
\cite{zakharov2019dpod} estimates a dense 2D-3D correspondence map between the input image and object model.
Although our approach follows the approach from the third category, our task focuses on a more general setting where the object models are not available during inference.

\noindent\textbf{Category-Level Object Detection.}
The task of 3D object detection aims to estimate 3D bounding boxes of objects in the scene.
\cite{song2016deep} runs sliding windows in 3D space and generates amodal proposals for objects.
\cite{gupta2015aligning, lahoud20172d, qi2018frustum} first generate 2D object proposals and then lift the proposals to 3D space.
\cite{yang2018pixor, zhou2018voxelnet} are single-stage detectors which directly detect objects from 3D data.
Although the above mentioned methods address the problem at category-level, the considered objects are usually constrained to the ground surface, e.g. instances of typical furniture classes in indoor scenes, cars, pedestrians, and cyclists in outdoor scenes.
Consequently, the assumption that rotation is constrained to be only along the gravity direction has to be made.
On the contrary, our approach can recover the full 6D pose of objects.

\noindent\textbf{Category-Level Pose Estimation.}
There are only a few pioneering works focusing on estimating 6D pose of unseen objects.
\cite{burchfiel2019probabilistic} leverages a generative representation of 3D objects and produces a multimodal distribution over poses with mixture density network.
However, only rotation is considered in their work.
\cite{sahin2018category} introduces a part-based random forest which employs simple depth comparison features, but our approach deals with RGB-D images.
\cite{wang2019normalized} proposes a canonical representation for all instances within an object category.
Our approach also makes use of this representation.
Instead of directly regressing the coordinates in NOCS, we account for intra-class shape variation by explicitly modeling the deformation from shape prior to object model.
\cite{chen2020learning} trains a variational autoencoder to generate the complete object model.
However, the reconstructed shape is not utilized for pose estimation.
In our network, shape reconstruction and pose estimation are integrated together.
\cite{wang20196} proposes the first category-level pose tracker, while our approach performs pose estimation without using temporal information.

\noindent\textbf{Shape Deformation.}
3D deformation is commonly applied for object reconstruction from a single image.
\cite{yumer2016learning, pontes2018image2mesh, kurenkov2018deformnet} use free-form deformation in conjunction with voxel, mesh and point cloud representations, respectively.
\cite{wang2018pixel2mesh, wen2019pixel2mesh++} 
starts from a coarse shape and
predicts a series of deformations to progressively improve the geometry.
Similar to us, \cite{wang20193dn} also supervise the deformation with global feature of the target.
However, we circumvent the fixed topology assumption of mesh representation by using point clouds instead.

\section{Our Method}

\noindent\textbf{Background.} Given an RGB-D image as the input, our goal is to detect and localize all visible object instances in the 3D space.
The object instances are not seen previously, but must come from known categories.
Each object is represented by a class label and an amodal 3D bounding box parameterized by its 6D pose and size.
The 6D pose is defined to be the rigid transformation (i.e. rotation and translation) that transforms the object from the reference to the camera coordinate frame.
It is common to choose the coordinate frame of the given 3D object models as reference
in instance-level 6D object pose estimation. Unfortunately, this is not viable for our category-level task since the instances of the 3D models are not available.
To mitigate this problem, we leverage on the Normalized Object Coordinate Space (NOCS) -- a shared canonical representation for all possible object instances within a category proposed in \cite{wang2019normalized}. 
The categorical 6D object pose and size estimation problem is then reduced to finding the similarity transformation 
between the observed depth map of each object instance and its corresponding points in the NOCS (i.e. \textit{NOCS coordinates}).

\begin{figure}[t]
    \centering
    \includegraphics[width=\columnwidth]{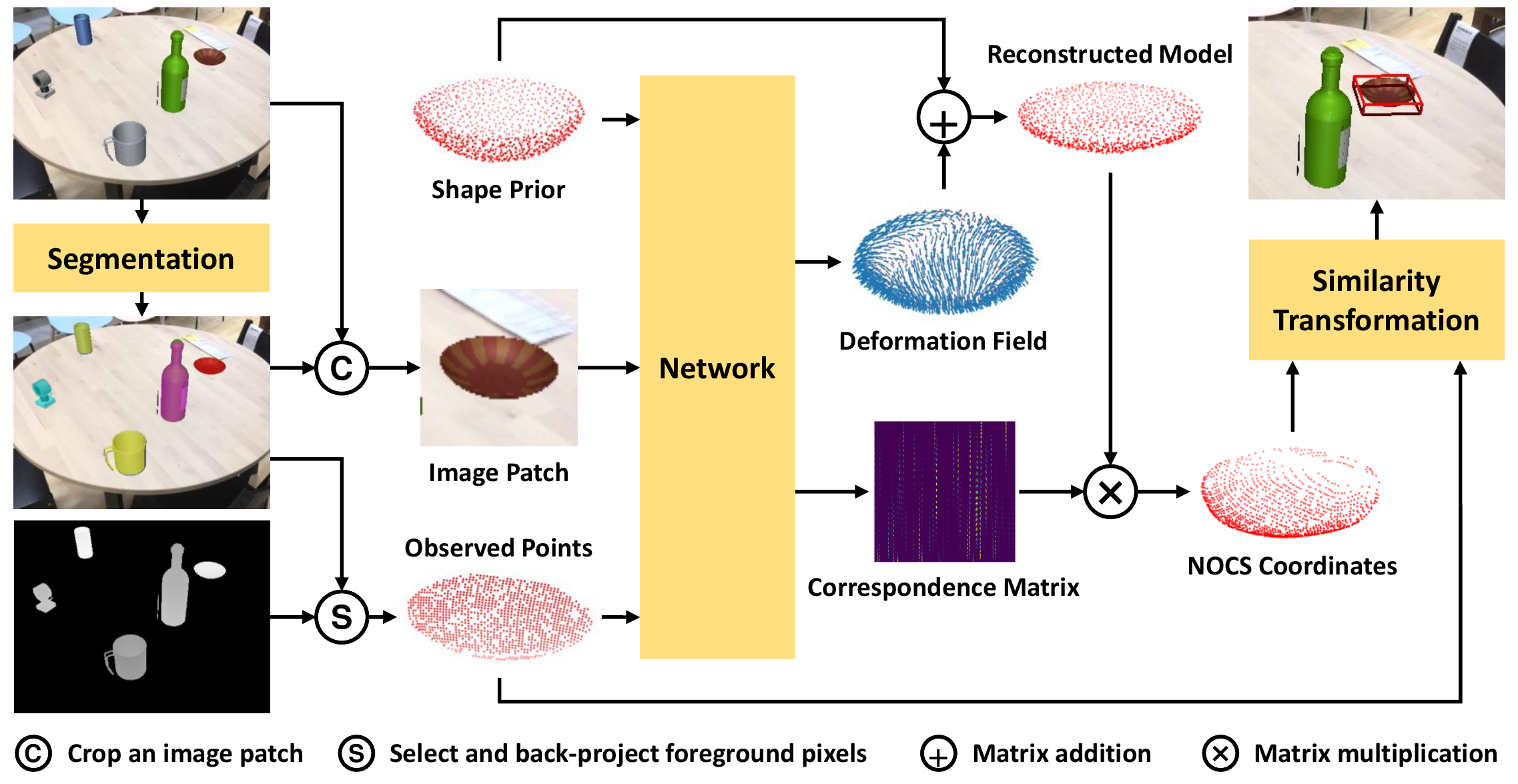}
    \caption{
    \textbf{Overview of our approach.}
    We first obtain a foreground mask for each object instance.
    Next our network reconstructs the instance (\textit{bowl} as an example) and establishes the correspondences between the observed points and the reconstructed model.
    Finally, the 6D pose is recovered by estimating a similarity transformation.
    Refer to Fig. \ref{fig:network} for the details of our network.
    }
    \label{fig:overview}
\end{figure}

\noindent\textbf{Overview.} 
In contrast to \cite{wang2019normalized} that directly outputs the NOCS coordinates from a Convolutional Neural Network (CNN), we propose an intermediate step to estimate the deformation of a pre-learned shape prior to improve the learning of intra-class shape variation. Our shape priors are learned from a collection of models spanning all categories (Section \ref{sect:prior}).
As shown in Fig.~\ref{fig:overview}, our approach consists of three stages.
The first stage performs instance segmentation on color image using an off-the-shelf network (e.g. Mask R-CNN \cite{he2017mask}).
Next we convert the masked depth map into a point cloud with camera intrinsic parameters for each instance and crop an image patch according to the bounding box of the mask.
Taking the point cloud, image patch, and the corresponding shape prior as inputs, our network outputs a deformation field that deforms the shape prior into the shape of the desired object instance (a.k.a. reconstructed model).
Furthermore, our network outputs a set of correspondences that associates 
each point in the point cloud obtained from the observed depth map of the object instance with the points of the reconstructed model.
This set of correspondences is used to mask the reconstructed model into the NOCS coordinates
(Section \ref{sect:network}).
Finally, the 6D pose and size of the object can be estimated by registering the NOCS coordinates and the point cloud obtained from the observed depth map 
(Section \ref{sect:registration}).

\subsection{Categorical Shape Prior}
\label{sect:prior}
Although object shape varies among different instances, an investigation over the 3D models reveals that objects of the same category (especially artificially generated objects) tend to have semantically and geometrically similar components.
For example, cameras are usually made up of a nearly cuboid body and a cylindrical lens; and mugs are typically cylindrical with a handle.
These categorical characteristics provide strong priors on the shape reconstruction of novel instances.
We propose the learning of a mean shape to capture the high-level characteristics from all the available models for each respective category.
To this end, we first train an autoencoder with all available object models and then compute the mean latent embedding of each object category with the encoder. These latent embeddings are passed into the decoder to get the mean shape priors for each object category.  
Unlike methods such as simple averaging \cite{wallace2019few} and principal component analysis (PCA) \cite{burchfiel2019probabilistic} that operate on voxel representations, our autoencoder framework can be easily altered to take any 3D representations. 

\begin{figure}[t]
    \centering
    \includegraphics[width=\columnwidth]{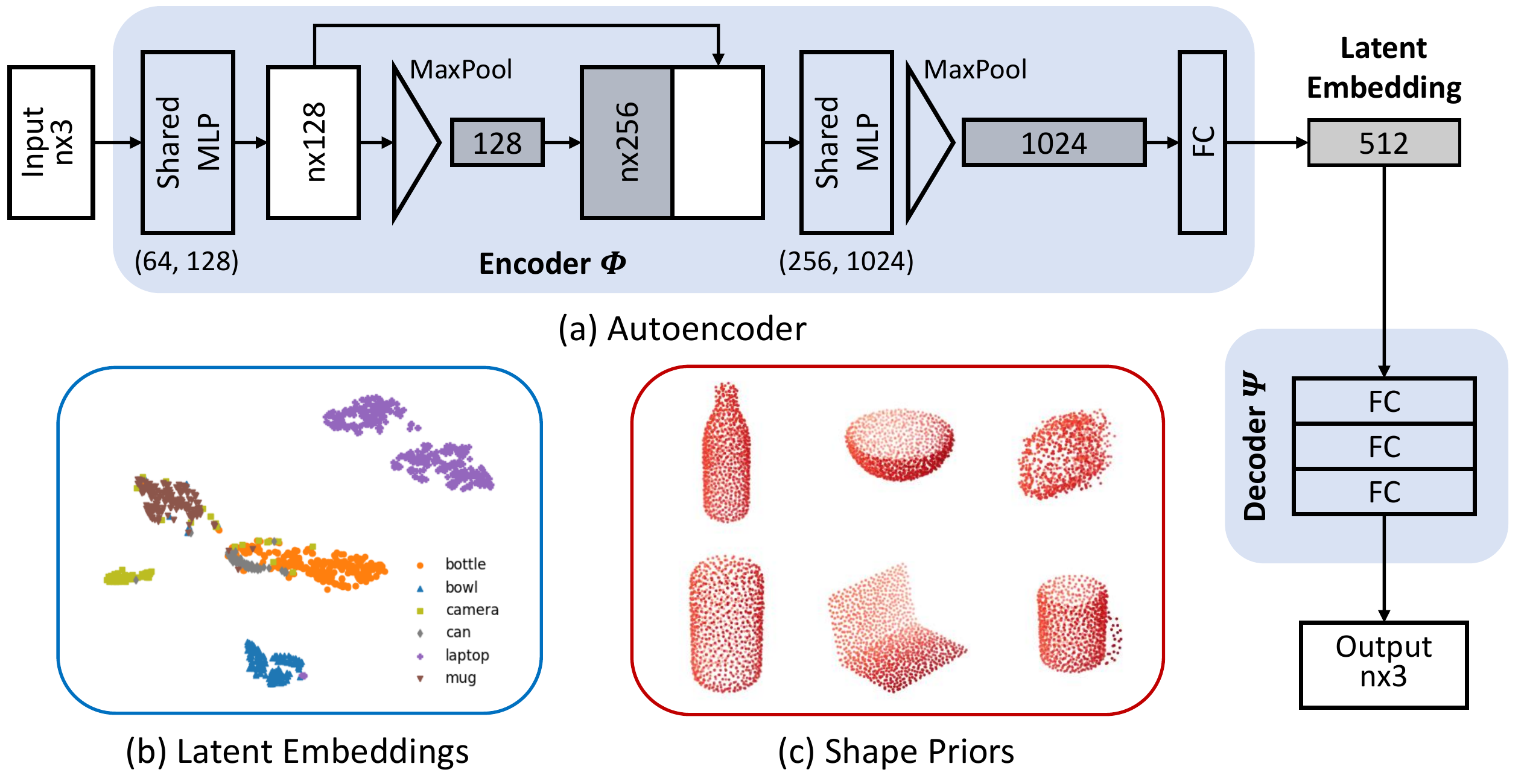}
    \caption{
    (a) Architecture of the autoencoder. (b) The latent embeddings of all instances are mapped to $\mathcal{R}^2$ with T-SNE for visualization. These instances are from  6 categories - \textit{bottle}, \textit{bowl}, \textit{camera}, \textit{can}, \textit{laptop} and \textit{mug}. (c) Shape priors are reconstructed by passing mean latent embedding of each category through the decoder.}
    \label{fig:shape_priors}
\end{figure}

Given a shape collection $\mathcal{M} = \{M_c^i \, \mid \, i=1, 2, \cdots, N; \, c = 1, 2, \cdots, C\}$, where  $M_c^i$ is the 3D point cloud model of instance $i$ from category $c$, we independently apply a similarity transformation to each model such that it is properly aligned in the NOCS.
This step ensures that the learned shape prior has the same scale and orientation as the target shape to be reconstructed.
The encoder $\Phi$ takes the point cloud and outputs a low-dimensional feature vector, i.e. the latent embedding $z_c^i \in \mathcal{R}^n$.
The decoder $\Psi$ takes this feature vector and outputs a point cloud that reconstructs the input:
\begin{equation}
    \hat{M_c^i} = (\Psi \circ \Phi)(M_c^i) = \Psi (z_c^i).
\end{equation}
Specifically, we adopt the PointNet-like encoder proposed in \cite{yuan2018pcn}, and a three-layer fully-connected decoder as shown in Fig. \ref{fig:shape_priors}a.
The reconstruction error is measured by the Chamfer distance:
\begin{equation}
\label{eq:cd_loss}
    d_{\text{CD}}(M_c^i, \hat{M_c^i}) = \sum_{x \in M_c^i} \min_{y \in \hat{M_c^i}} \| x - y \|_2^2 + \sum_{y \in \hat{M_c^i}} \min_{x \in M_c^i} \| x - y \|_2^2 .
\end{equation}

The autoencoder is trained on a shape collection by minimizing the reconstruction error.
Once the training is converged, we obtain the latent embeddings $\{ z_c^i \}$ of all instances in $\mathcal{M}$.
Although not explicitly enforced during training, these latent embeddings form clusters in the latent space according to their categories.
Fig. \ref{fig:shape_priors}b visualizes the clustering effect of the embeddings.
We use T-SNE \cite{maaten2008visualizing} to further embed these features in $\mathcal{R}^2$ for visualization.
Similar clustering results are also observed on a different set of models \cite{yang2018foldingnet}.
Based on this observation, we compute the mean latent embedding for each category and then pass it through the decoder to construct the shape prior:
\begin{equation}
    M_c = \Psi (z_c) = \Psi \left( \frac{1}{N_c} \sum_{i} z_c^i \right).
\end{equation}
The resulting categorical shape priors $\{ M_c \}$ are shown in Fig. \ref{fig:shape_priors}c.

\subsection{Our Network Architecture}
\label{sect:network}
We denote the observation of an object instance as $(V, I)$, where $V \in \mathcal{R}^{N_v \times 3}$ is the point cloud and $I \in \mathcal{R}^{H \times W \times 3}$ is the image patch.
$N_v$ denotes the number of foreground pixels with a valid depth value.
The corresponding shape prior is $M_c \in \mathcal{R}^{N_c \times 3}$, where $N_c$ is the number of points in $M_c$.
Our network takes $V$, $I$ and $M_c$ as inputs, and outputs the per-point deformation field $D \in \mathcal{R}^{N_c \times 3}$ and a correspondence matrix $A \in \mathcal{R}^{N_v \times N_c}$.
The final reconstructed model is $M = M_c + D$.
Each row of $A$ sums to 1 since it represents the soft correspondences between a point in $V$ and all points in $M$.
As shown in Fig. \ref{fig:network}, our network is composed of four parts: (1) extracts features from the object instance; (2) extracts features from the shape prior; (3) regresses the deformation field $D$; and (4) estimates the correspondence matrix $A$.

\begin{figure}[t]
    \centering
    \includegraphics[width=\columnwidth]{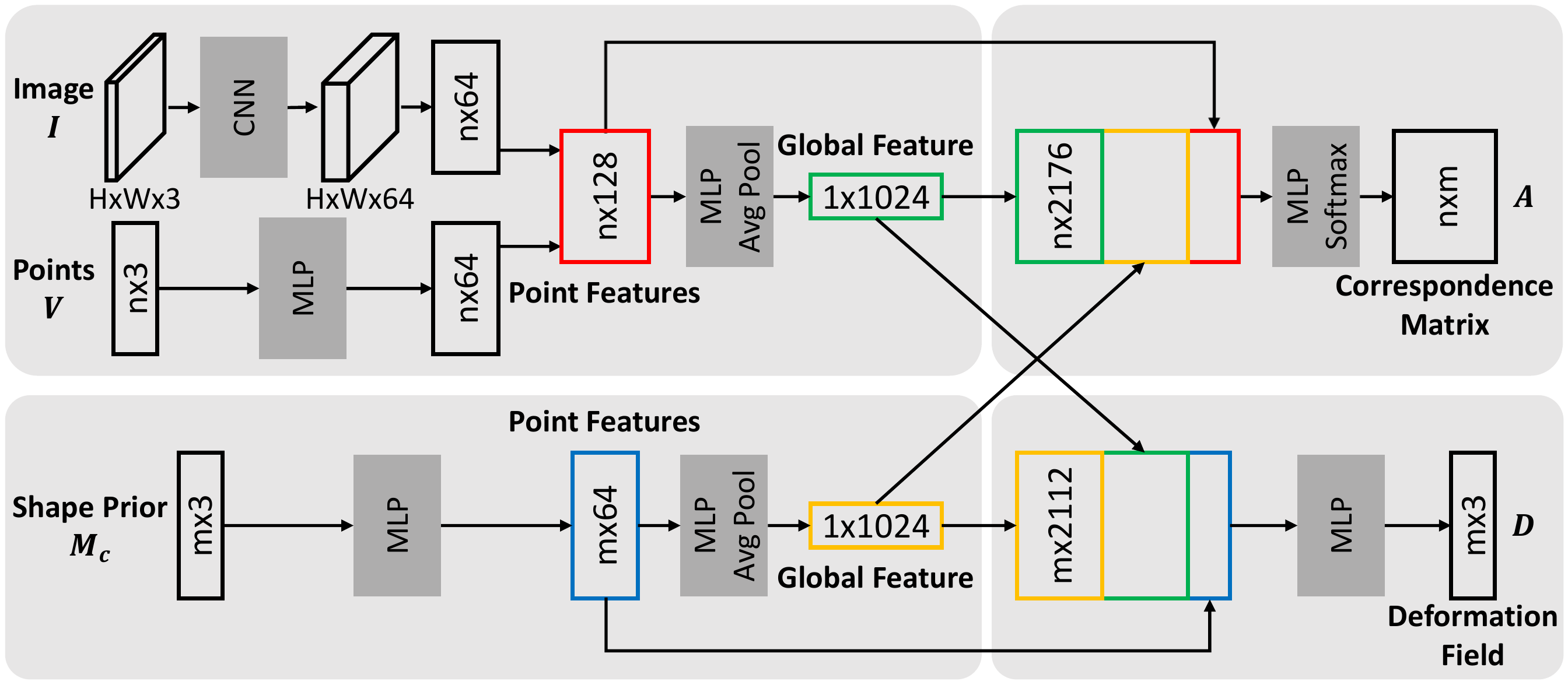}
    \caption{
    \textbf{Our Network Architecture.}
    The upper-left and lower-left branches extract point and global features from the instance and the shape prior respectively.
    The upper-right branch estimates the correspondence matrix, and the lower-right branch predicts the deformation field.
    The exchange of global features is the key part of our network.
    }
    \label{fig:network}
\end{figure}

On the consideration that the depth and color are two different modalities, we follow the pixel-wise dense fusion approach proposed in \cite{wang2019densefusion} to effectively extract RGB-D features from the observation.
For point cloud $V$, we use an embedding network similar to PointNet \cite{qi2017pointnet} to generate per-point geometric features by mapping each point in $V$ to the $d_g$-dimensional feature space.
The image patch $I$ is processed with a fully convolutional network which follows an encoder-decoder architecture and maps $I$ to $\mathcal{R}^{H \times W \times d_c}$.
Next we associate the geometric feature of each point with its corresponding color feature and concatenate the feature pairs.
Since each point in $V$ has a corresponding pixel in $I$, not vice versa, redundant color features are discarded.
The concatenated features are termed ``instance point features" and fed to another shared multi-layer perceptron.
An average pooling layer is used to generate the ``instance global feature".
The categorical shape prior $M_c$ is a point cloud with purely geometric information.
We apply a simpler embedding network to extract the ``category point features" and the ``category global feature".

The shape prior $M_c$ provides the prior knowledge of the category, i.e. the coarse shape geometry and canonical pose.
Although the observation $(V, I)$ is partial, it provides instance-specific details of the target shape.
A natural way to reconstruct the object in NOCS is to deform $M_c$ under the guidance of $(V, I)$.
Consequently, we concatenate the category and instance global features, and enrich the category point features with the concatenated features.
The obtained feature vectors are successively convolved with  $1 \times 1$ kernels to generate the deformation field $D$.
Similar intuition and feature concatenation strategy also apply to the estimation of $A$.
We combine the instance point features and global feature to aggregate both local and global information for each point.
Each point in $V$ gets mapped to the points of the reconstructed model
through concatenation with the category global feature.
We obtain the NOCS coordinates, denoted as $P$, of the points in $V$
by multiplying $A$ and $M$, i.e.
\begin{equation}
\label{eq:noc}
    P = A \times M = A (M_c + D) \in \mathcal{R}^{N_v \times 3}.
\end{equation}

\subsection{6D Pose Estimation}
\label{sect:registration}
Our goal is to estimate the 6D pose and size of the object instance.
Given depth observation $V$ and its NOCS coordinates $P$, the optimal similarity transformation parameters (rotation, translation, and scaling) can be computed by solving the absolute orientation problem using Umeyama algorithm \cite{umeyama1991least}.
We also implement the RANSAC algorithm \cite{fischler1981random} for robust estimation.

\subsection{Loss Functions}
\label{sect:loss}
In this section, we define the loss functions used to train our network, and explain how we handle object symmetry during training.

\noindent\textbf{Reconstruction Loss.}
We assume that ground-truth model $M_{gt}$ is available during training.
The deformation field $D$ is supervised indirectly by minimizing the Chamfer distance (c.f. Eq. \ref{eq:cd_loss}) between $M$ and $M_{gt}$, i.e. $ L_{\text{cd}} = d_{\text{CD}}(M, M_{gt}) = d_{\text{CD}}(M_c + D, M_{gt}) $.

\noindent\textbf{Correspondence Loss.}
It is impractical and unnecessary to pre-compute the ground-truth value for $A$.
Alternatively, we supervise $A$ indirectly via the NOCS coordinates $P$ (which is a result of applying the correspondence matrix on the reconstructed model) since the ground-truth NOCS coordinates $P_{gt}$ can be obtained easily from the object model and its 6D pose through image rendering.
We use the smooth $L_1$ loss function:
\begin{equation}
    L_{\text{corr}} (P, P_{gt}) = \frac{1}{N_v} \sum_{\mathbf{x} \in P} \sum_{i=1,2,3}
        \begin{cases}
        5(x_i - y_i)^2 , & \text{if} \; |x_i - y_i| \leq 0.1 \\
        |x_i - y_i| -0.05 , & \text{otherwise}
    \end{cases},
\end{equation}
where $\mathbf{x} = (x_1, x_2, x_3) \in P$, and $\mathbf{y} = (y_1, y_2, y_3) \in P_{gt}$.

Object symmetry is an inevitable problem for pose estimation algorithms, especially for those that require supervised training.
For symmetrical objects, there exists at least one rotation such that appearance of the object is preserved under this rotation.
In other words, two observations of a symmetric object can be very similar but with different rotation labels.
We follow the solution proposed by \cite{pitteri2019object} to map ambiguous rotations to a canonical one.
More specifically, the Map operator for an arbitrary rotation $R \in SO(3)$ is defined as:
\begin{equation}
    \text{Map}(R) = R \hat{S}, \: \text{with} \; \hat{S} = \underset{S \in \mathcal{S}(M_c^i)}{\arg\min} {\| R S - I_3 \|_F},
\end{equation}
where the proper symmetry group $ \mathcal{S}(M_c^i) $ is the set of rotations which preserve the appearance of a given object $M_c^i$.
The experimental datasets assume continuous symmetry and the axis of symmetry is the y-axis of the NOCS.
Hence, $\hat{S}$ takes the following form:
\begin{equation}
    \hat{S} =
        \setlength{\arraycolsep}{5pt}
        \begin{bmatrix}
        \cos \hat{\theta} & 0 & - \sin \hat{\theta} \\
        0 & 1 & 0 \\
        \sin \hat{\theta} & 0 & \cos \hat{\theta}
        \end{bmatrix}, \;
    \text{with} \; \hat{\theta} = \arctan 2 (R_{13} - R_{31}, R_{11} + R_{33}),
\end{equation}
where $R_{11}$, $R_{13}$, $R_{31}$, and $R_{33}$ are the elements of $R$.
During training, we apply the Map operator to the rotation label: $R_{gt} \leftarrow R_{gt} \hat{S}$ to eliminate the rotation ambiguity
of any symmetric object with the ground-truth pose $(R_{gt}, T_{gt})$. In practice, our network is supervised by ground-truth NOCS coordinates $P_{gt}$.
Equivalently, we transform $P_{gt}$ by $\hat{S}^T$: $P_{gt} \leftarrow \hat{S}^T P_{gt}$.

\noindent\textbf{Regularization Losses.}
Row $A_i$ of the matrix $A$ represents the distribution over the correspondences between $i$-th point of $V$ and the points in $M$.
$A_i$ can be understood as a relaxed one-hot vector, since each point of $V$ usually can be well approximated by at most three points of $M$.
We encourage $A_i$ to be a peaked distribution by minimizing the average cross entropy: $L_{\text{entropy}} = \frac{1}{N_v} \sum_{i}\sum_{j} - A_{i,j} \log A_{i,j}$.
We also regularize $D$ to discourage large deformations: $L_{\text{def}} = \frac{1}{N_C} \sum_{\mathbf{d}_i \in D} \| \mathbf{d}_i \| _2$.
Minimal deformation preserves the semantic consistency between shape prior and the reconstructed model.
For example, we want that the point belongs to the handle of the ``mug" prior remains in the handle after deformation.
This consistency loss is beneficial for the prediction of the correspondence matrix $A$.

In summary, the overall objective is a weighted sum of all four losses:
\begin{equation}
    L = \lambda_{1} L_{\text{cd}} + \lambda_{2} L_{\text{corr}} + \lambda_{3} L_{\text{entropy}} + \lambda_{4} L_{\text{def}} .
\end{equation}

\section{Experiments}

\subsection{Experimental Setup}

\noindent\textbf{Datasets.}
The CAMERA \cite{wang2019normalized} dataset is generated by rendering and compositing synthetic objects into real scenes in a context-aware manner.
In total, there are 300K composite images, where 25K are set aside for evaluation.
The training set contains 1085 object instances selected from 6 different categories - \textit{bottle}, \textit{bowl}, \textit{camera}, \textit{can}, \textit{laptop} and \textit{mug}.
The evaluation set contains 184 different instances.
The REAL \cite{wang2019normalized} dataset is complementary to the CAMERA.
It captures 4300 real-world images of 7 scenes for training, and 2750 real-world images of 6 scenes for evaluation.
Each set contains 18 real objects spanning the 6 categories.
The two evaluation sets are referred to as CAMERA25 and REAL275.

\noindent\textbf{Evaluation Metric.}
Following \cite{wang2019normalized}, we independently evaluate the performance of 3D object detection and 6D pose estimation.
We report the average precision at different Intersection-Over-Union (IoU) thresholds for object detection.
For 6D pose evaluation, the average precision is computed at $\text{n} ^\circ \, \text{m} cm$.
We ignore the rotational error around the axis of symmetry 
for symmetrical object categories (e.g. \textit{bottle}, \textit{bowl}, and \textit{can}).
Specially, we treat \textit{mug} as symmetric object in the absence of the handle, and asymmetric object otherwise.

\noindent\textbf{Baseline.}
Wang et al. \cite{wang2019normalized} is currently the only publicly available code and datasets for the 6D object pose and size estimation task. Futhermore, it is also the state-of-the-art performance on the task. 
Hence, we choose \cite{wang2019normalized} as our baseline for comparison.

\subsection{Implementation Details}
We collect all the instances in the CAMERA training dataset to train the autoencoder.
Shape priors are learned from this collection and used in all experiments.
Each prior consists of 1024 points.
We use the Mask R-CNN implemented by matterport \cite{matterport_maskrcnn_2017} for instance segmentation.
For each detected instance, we resize the image patch to $192 \times 192$, and randomly sample 1024 points by repetition (if insufficient point count) or downsampling (if sufficient point count).
To extract instance color features, we choose the PSPNet \cite{zhao2017pyramid} with ResNet-18 \cite{he2016deep} as backbone.
We randomly select 5 point-pairs to generate a hypothesis for the RANSAC-based pose fitting.
The maximum number of iteration is 128 and inlier threshold is set to 10\% of the object diameter.
For the hyperparameters of the total loss, we empirically find that $\lambda_{1} = 5.0$, $\lambda_{2} = 1.0$, $\lambda_{3} = 1e-4$, and $\lambda_{4} = 0.01$ are good choices.

\subsection{Comparison to Baseline}

We compare our approach to the Baseline \cite{wang2019normalized} on CAMERA25 and REAL275.
Quantitative results are summarized in Table \ref{table:map_comparison}.

\noindent\textbf{CAMERA25.}
In the setting of estimating 6D object pose and size from an RGB-D image, we achieve a mAP of 83.1\% for 3D IoU at 0.75, and a mAP of 54.3\% for 6D pose at $ 5^\circ \, 2 \text{cm} $.
Our results are 14\% and 22\% higher than the Baseline \cite{wang2019normalized}, respectively.
We naively remove the depth input and related sub-networks in our network (i.e. RGB image as the only input) to make fair comparisons with the Baseline \cite{wang2019normalized}, which takes an RGB image as its input.
As shown in Table~\ref{table:map_comparison}, our results without depth input are still significantly better than the Baseline \cite{wang2019normalized} (i.e. +15.5\% and +17.9\%).
On one hand, this experiment shows the advantage of explicit handling of the intra-class shape variation, and the effectiveness of our method which reconstructs the object via deformation.
On the other hand, it also shows that adding depth to the network does help to improve overall performance, although our improved performance does not rely on it solely.
Given that depth image is required to uniquely determine the scale of the object, we recommend it in practical applications.
The top row of Fig. \ref{fig:map} shows the average precision at different error thresholds for all 6 object categories.
It provides independent analysis for 3D IoU, rotation, and translation error.

\begin{table}[t]
\centering
\caption{
    Comparisons on CAMERA25 and REAL275.
    We report the mAP w.r.t. different thresholds on 3D IoU, and rotation and translation errors.
    }
\label{table:map_comparison}
\setlength{\tabcolsep}{5pt}
\begin{adjustbox}{width=\columnwidth}
\begin{tabular}{ c | c | c  c  c  c  c  c }
\toprule
\multirow{2}{*}{Data} & \multirow{2}{*}{Method} & \multicolumn{6}{c}{mAP} \\
\cline{3-8}
& & $3\text{D}_{50}$ & $3\text{D}_{75}$ & $5^\circ \, 2 \text{cm}$ & $5^\circ \, 5 \text{cm}$ & $10^\circ \, 2 \text{cm}$ & $10^\circ \, 5 \text{cm}$ \\
\midrule
\multirow{3}{*}{CAMERA25}  & {Baseline \cite{wang2019normalized}} &  83.9  &  69.5  &  32.3  &  40.9  &  48.2  &  64.6  \\
& {Ours (RGB)} &  93.1  &  \textbf{84.6}  &  50.2  &  54.5  &  70.4  &  78.6  \\
& {Ours (RGB-D)} &  \textbf{93.2}  &  83.1  &  \textbf{54.3}  &  \textbf{59.0}  &  \textbf{73.3}  &  \textbf{81.5}  \\
\midrule
\multirow{3}{*}{REAL275}  & {Baseline \cite{wang2019normalized}}  &  \textbf{78.0}  &  30.1  &  7.2  &  10.0  &  13.8  &  25.2  \\
& {Ours (RGB)} &  75.2  &  46.5  &  15.7  &  18.8  &  33.7  &  47.4  \\
& {Ours (RGB-D)} &  77.3  &  \textbf{53.2}  &  \textbf{19.3}  &  \textbf{21.4}  &  \textbf{43.2}  &  \textbf{54.1} \\
\bottomrule
\end{tabular}
\end{adjustbox}
\end{table}

\noindent\textbf{REAL275.}
The REAL training set only contains 3 object instances per category, we enlarge this training set such that the network can generalize well to unseen objects.
Following the Baseline \cite{wang2019normalized}, we randomly select data from CAMERA and REAL training set according to a ratio of $3:1$.
In fair comparison to the Baseline \cite{wang2019normalized}, our approach improves the mAP by 23.1\% for 3D IoU at 0.75 and 12.1\% for 6D pose at $ 5^\circ \, 2 \text{cm} $.
In strict comparison, we can still outperform by 16.4\% and 8.5\%, respectively.
These results provide  further evidence to support our approach.
Fig. \ref{fig:map} (bottom) shows more detailed analysis of the errors.

\begin{figure}[t]
    \centering
    \includegraphics[width=\columnwidth]{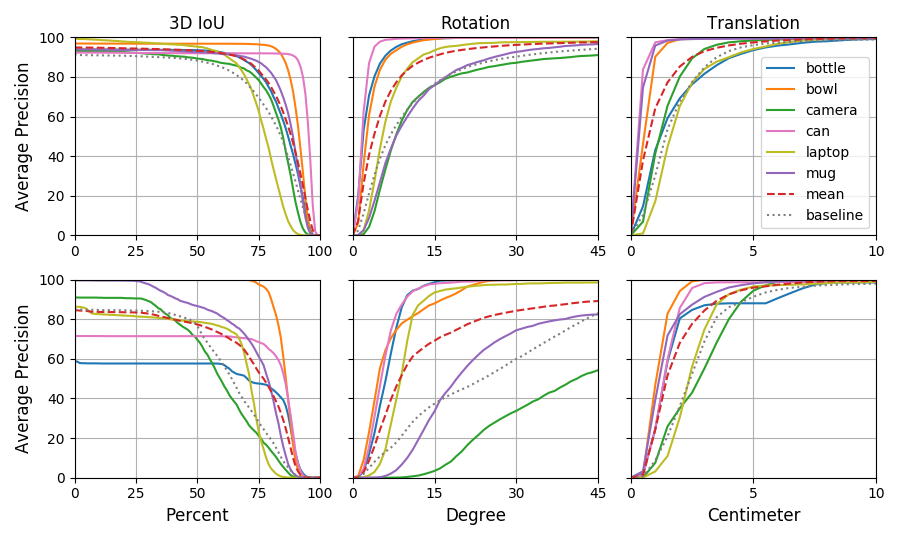}
    \caption{Average precision vs. error thresholds on CAMERA25 (top row) and REAL275 (bottom row).}
    \label{fig:map}
\end{figure}

\subsection{Evaluation of Shape Reconstruction}

\begin{table}[t]
\centering
\caption{Evaluation of shape reconstruction with CD metric ($\times 10^{-3}$).}
\label{table:shape_evaluation}
\setlength{\tabcolsep}{3pt}
\begin{adjustbox}{width=\columnwidth}
\begin{tabular}{ c | c | c  c  c  c  c  c  c}
\toprule
{Data}  &  {Model}  &  {Bottle}  &  {Bowl}
& {Camera} &  {Can}  &  {Laptop}  &  {Mug}  &  {Average}  \\
\midrule
\multirow{2}{*}{CAMERA25}  &  {Reconstruction}  &  1.81  &  1.63  &  4.02  &   0.97  &  1.98  &  1.42  &  1.97  \\
 &  {Shape Prior}  &  3.41  &  2.20  &  9.01  &  2.21  &  3.27  &  2.10  &  3.70  \\
\midrule
\multirow{2}{*}{REAL275}  &  {Reconstruction}  &  3.44  &  1.21  &  8.89  &  1.56  &  2.91  &  1.02  &  3.17  \\
 &  {Shape Prior}  &  4.99  &  1.16  &  9.85  &  2.38  &  7.14  &  0.97  &  4.41  \\
\bottomrule
\end{tabular}
\end{adjustbox}
\end{table}

To evaluate the quality of the reconstruction, we compute the CD metric (c.f. Eq. \ref{eq:cd_loss}) of the reconstructed model from our method with the ground truth model in the NOCS.
We get a CD metric of 1.97 on CAMERA25 and 3.17 on REAL275. In comparison, the CD metrics are 3.70 and 4.41 on the respective dataset for the shape priors from our autoencoder.
The better CD metrics of the reconstructed models compared to the shape priors show that the deformation estimation in our framework improves the quality of the 3D model reconstruction.
Table \ref{table:shape_evaluation} shows the CD metric of our reconstructed models and the shape priors for each category.

\subsection{Ablation Studies}

\noindent\textbf{Different shape priors.}
We first evaluate how different shape priors influence the performance.
All settings are kept the same in this experiment except for the choices of the priors.
Results are summarized in Table \ref{table:ablation_camera} and \ref{table:ablation_real}.
``Embedding" refers to the priors obtained from decoding the mean latent embeddings.
We also try the instance whose latent embedding has the minimum $L_2$ distance to the mean latent embedding (denoted as ``NN").
In addition, we explore random selection of one instance per category from the shape collection to compose our priors (denoted as ``Random").
In general, our approach remains stable under different priors.
Our network can adapt to different shape priors because the deformation is explicitly estimated.
We achieve the best result for accurate pose (i.e. $5^\circ \, 2 \text{cm}$) estimation when the learned categorical shape prior is used.
Since our main target is to recover the 6D pose, we choose ``Embedding" as our best model.
To validate whether the priors are necessary, we use a point cloud uniformly sampled from a sphere of diameter one as our prior (denoted as ``None").
The mAP decreases by 3.7\% on real dataset
when there are no priors, but the best result is achieved for object size estimation.
Although shape priors are beneficial for estimating 6D pose, they sometimes bias shape reconstruction.

\begin{table}[t]
\centering
\caption{Ablation studies on CAMERA25. Refer to text for more details.}
\label{table:ablation_camera}
\setlength{\tabcolsep}{5pt}
\begin{adjustbox}{width=\columnwidth}
\begin{tabular}{ c | c | c  c  c  c  c  c }
\toprule
\multirow{2}{*}{} & \multirow{2}{*}{Network} & \multicolumn{5}{c}{mAP} \\
\cline{3-8}
& & $3\text{D}_{50}$ & $3\text{D}_{75}$ & $5^\circ \, 2 \text{cm}$ & $5^\circ \, 5 \text{cm}$ & $10^\circ \, 2 \text{cm}$ & $10^\circ \, 5 \text{cm}$ \\
\midrule
\multirow{4}{*}{Shape Priors}  & {Embedding} &  93.2  &  83.1  &  \textbf{54.3}  &  \textbf{59.0}  &  73.3  &  81.5  \\
& {NN} &  93.3  &  85.7  &  52.7  &  57.3  &  72.9  &  81.3  \\
& {Random} &  93.3  &  85.7  &  53.4  &  58.0  &  72.8  &  81.0 \\
& {None} &  \textbf{93.3}  &  \textbf{85.8}  &  54.0  &  58.8  &  73.1  &  81.6  \\
\midrule
\multirow{1}{*}{NOCS Coords}  & {Regression}  &  93.3  &  85.3  &  51.2  &  55.6  &  \textbf{73.8}  &  \textbf{82.1}  \\
\midrule
\multirow{2}{*}{Regularization}  & {w/o Def.}  &  93.2  &  85.1  &  53.9  &  58.7  &  73.1  &  81.4  \\
& { w/o Entropy} &  93.2  &  85.1  &  53.2  &  57.9  &  73.2  &  81.8  \\
\bottomrule
\end{tabular}
\end{adjustbox}
\end{table}

\begin{table}[t]
\centering
\caption{Ablation studies on REAL275. Refer to text for more details.}
\label{table:ablation_real}
\setlength{\tabcolsep}{5pt}
\begin{adjustbox}{width=\columnwidth}
\begin{tabular}{ c | c | c  c  c  c  c  c }
\toprule
\multirow{2}{*}{} & \multirow{2}{*}{Network} & \multicolumn{5}{c}{mAP} \\
\cline{3-8}
& & $3\text{D}_{50}$ & $3\text{D}_{75}$ & $5^\circ \, 2 \text{cm}$ & $5^\circ \, 5 \text{cm}$ & $10^\circ \, 2 \text{cm}$ & $10^\circ \, 5 \text{cm}$ \\

\midrule
\multirow{4}{*}{Shape Priors}  & {Embedding} &  77.3  &  53.2  &  \textbf{19.3}  &  \textbf{21.4}  &  \textbf{43.2}  &  \textbf{54.1}   \\
& {NN} &  75.9  &  52.6  &  17.0  &  19.0  &  42.0  &  51.6  \\
& {Random} &  75.8  &  52.2  &  17.9  &  20.1  &  42.3  &  51.3 \\
& {None} &  77.2  &  \textbf{55.5}  &  15.6  &  19.8  &  38.4  &  53.6 \\
\midrule
\multirow{1}{*}{NOCS Coords}  & {Regression}  &  \textbf{78.7}  &  54.9  &  13.7  &  14.9  &  42.5  &  51.4  \\
\midrule
\multirow{2}{*}{Regularization}  & {w/o Def.}  &  77.1  &  50.2  &  13.4  &  15.4  &  37.3  &  49.8 \\
& {w/o Entropy} &  77.3  &  53.3  &  15.7  &  18.8  &  38.5  &  51.3 \\
\bottomrule
\end{tabular}
\end{adjustbox}
\end{table}

\noindent\textbf{Directly regress the NOCS Coordinates?}
As indicated by Eq. \ref{eq:noc}, our approach decouples the NOCS coordinates $P$ to shape reconstruction $M$ and dense correspondences $A$.
However, both the network architecture and the training will be much simpler when we follow \cite{wang2019normalized} to regress $P$ directly (denoted as ``Regression" in Table \ref{table:ablation_camera} and \ref{table:ablation_real}).
For 6D pose estimation, the mAP of ``Regression" at $5^\circ \, 2 \text{cm}$ is notably lower than ``Embedding" on CAMERA25 (-3.1\%) and REAL275 (-5.6\%).
This result further supports the benefit of handling shape variation via reconstruction over naive regression of the NOCS coordinates.
``Regression" achieves slightly better mAP for object size estimation since it only finds the NOCS coordinates for the observed part, while ``Embedding" needs to complete the unknown part of the object.

\noindent\textbf{Regularization losses.}
To validate the necessity of the two regularization losses, we train the network without regularizing deformation or correspondence, while keeping all the rest same as "Embedding".
The mean average precisions of both variants are still comparable to "Embedding" on synthetic dataset.
However, mAP of 6D pose at $5^\circ \, 2 \text{cm}$ drops noticeably (-5.9\% and -3.6\%)
on the more difficult real dataset.

\subsection{Qualitative Results.}
In Fig. \ref{fig:qualitative}, we provide several qualitative results from both synthetic and real instances.
The 6D pose and object size can be reliably recovered from noisy point correspondences using RANSAC-based pose fitting.
Shape reconstruction can capture the variations between instances.
The qualities of our predictions are generally better on synthetic data than real data, which is an indication that observation noise needs more attention in our future work.
Out of the six categories, \textit{camera} shows less accurate reconstruction due to its more complicated and varying geometry.

\begin{figure}[t]
    \centering
    \includegraphics[width=\columnwidth]{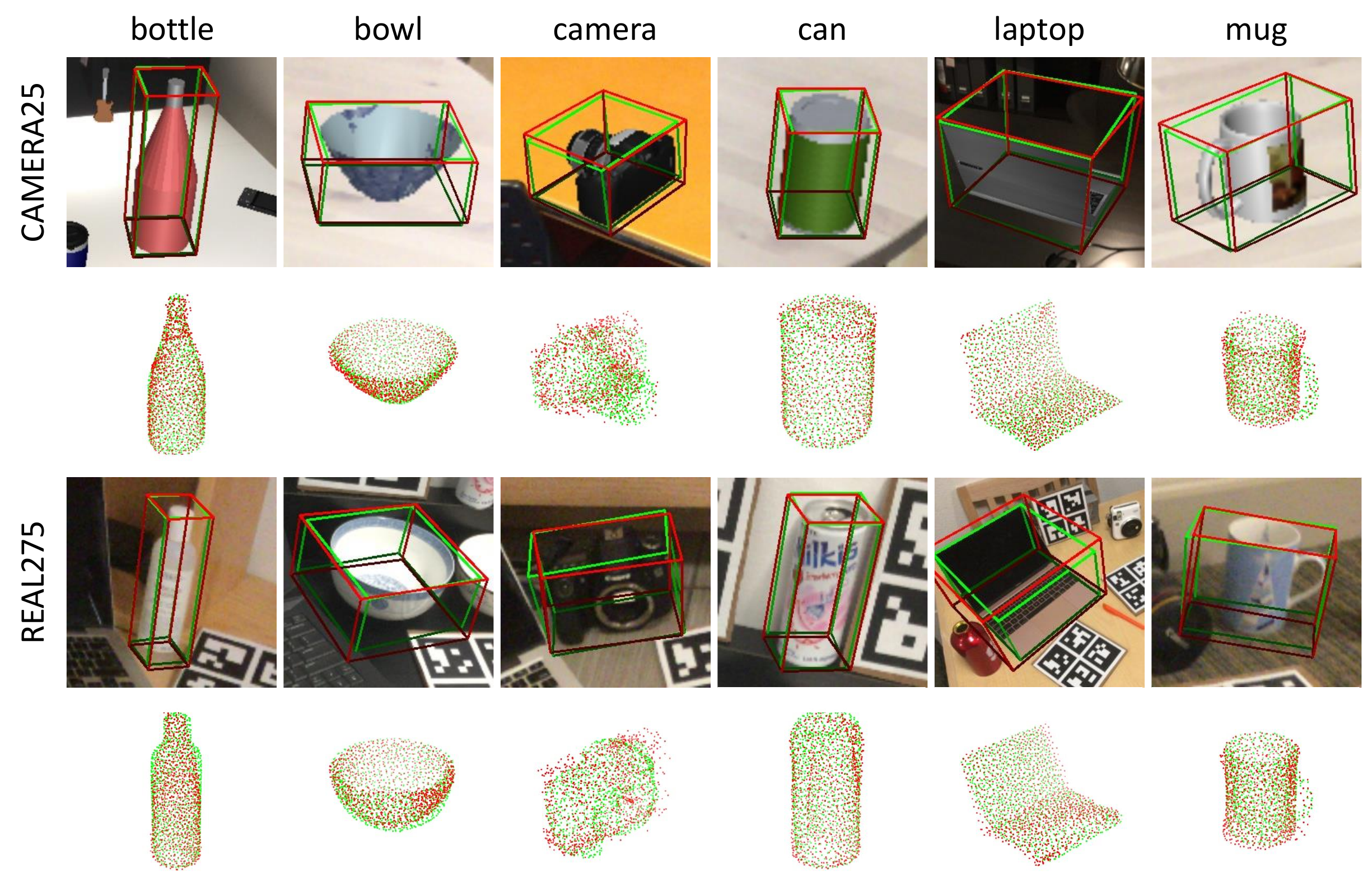}
    \caption{Examples of qualitative results from CAMERA25 (top rows) and REAL275 (bottom rows).
    For each example, we visualize the results of pose estimation and the reconstructed model $M$.
    Our estimations are shown in red, while the ground truths are shown in green.
    }
    \label{fig:qualitative}
\end{figure}

\section{Conclusions}
We present a novel approach for category-level 6D object pose estimation.
Our network 
explicitly models intra-class shape variation by the estimation of the deformation from a shape prior to the object model. Shape priors are learned form a collection of object models and constructed in the latent space.
Experiments on synthetic and real datasets demonstrate the advantage of our proposed approach.

\paragraph{\bf{Acknowledgements.}}
This research is supported in parts by the Singapore MOE Tier 1 grant
R-252-000-A65-114, and the Agency for Science, Technology and Research (A*STAR) under its AME Programmatic Funding Scheme (Project \#A18A2b0046).

\clearpage
%
%
\bibliographystyle{splncs04}
\bibliography{egbib}

\clearpage
\renewcommand\thesection{\Alph{section}}
\setcounter{section}{0}

\section{Comparison to CASS}
CASS \cite{chen2020learning} is the latest work on category-level 6D object pose and size estimation.
Similar to our work, they reconstruct the complete object model in the canonical space as a by-product.
However, they train a variational autoencoder to generate the point cloud, while we estimate the deformation field of the corresponding shape prior.
In addition, they directly regresses the pose and size by comparing pose-independent and pose-dependent features, while we recover the pose by establishing dense correspondences.
As shown in Table \ref{table:compare_cass}, our approach significantly outperforms CASS in pose accuracy. This demonstrates the superiority of our correspondence-based approach over direct pose regression.

\begin{table}[h]
\centering
\caption{Quantitative comparison with CASS on REAL275.}
\label{table:compare_cass}
\setlength{\tabcolsep}{4pt}
\begin{tabular}{ r | c  c  c  c  c  c  c }
\toprule
\multirow{2}{*}{\textbf{Method}} & \multicolumn{6}{c}{\textbf{mAP}} \\
\cline{2-8}
& $3\text{D}_{25}$ & $3\text{D}_{50}$ & $3\text{D}_{75}$ & $5^\circ \, 2 \text{cm}$ & $5^\circ \, 5 \text{cm}$ & $10^\circ \, 2 \text{cm}$ & $10^\circ \, 5 \text{cm}$ \\
\midrule
Baseline \cite{wang2019normalized}  &  \textbf{84.8}  &  \textbf{78.0}  &  30.1  &  7.2  &  10.0  &  13.8  &  25.2  \\
CASS \cite{chen2020learning} &  84.2 &  77.7  &  $-$  &  $-$  &  13.0  &  $-$  &  37.6  \\
Ours &  83.4  &  77.3  &  \textbf{53.2}  &  \textbf{19.3}  &  \textbf{21.4}  &  \textbf{43.2}  &  \textbf{54.1} \\
\bottomrule
\end{tabular}
\end{table}

\section{Comparison to 6-PACK}
6-PACK \cite{wang20196} is the state-of-the-art category-level 6D pose tracker.
Although our approach does not require pose initialization nor leverages on temporal consistency, we still achieve comparable accuracy on REAL275 at $5^\circ \, 5 \text{cm}$ (30.4\% compared to 33.3\%).
More importantly, the accuracy of 6-PACK drops below 30\% when the first 40 frames of a sequence (460 frames on average) are excluded from evaluation. This indicates that 6-PACK is highly dependent on pose initialization for higher accuracy. In contrast, the accuracy of our method remains stable since it is a pose estimation method.

\section{Qualitative Results}
Fig. \ref{fig:sup_qualitative} shows the per-frame pose detection results of our approach.
The results are better on synthetic data (top two rows) than on real data (bottom two rows).
This performance gap is mainly induced by the observation noise, which has greater influence on objects with complicated geometric shape (e.g. camera).

\begin{figure}[t]
    \centering
    \includegraphics[width=\columnwidth]{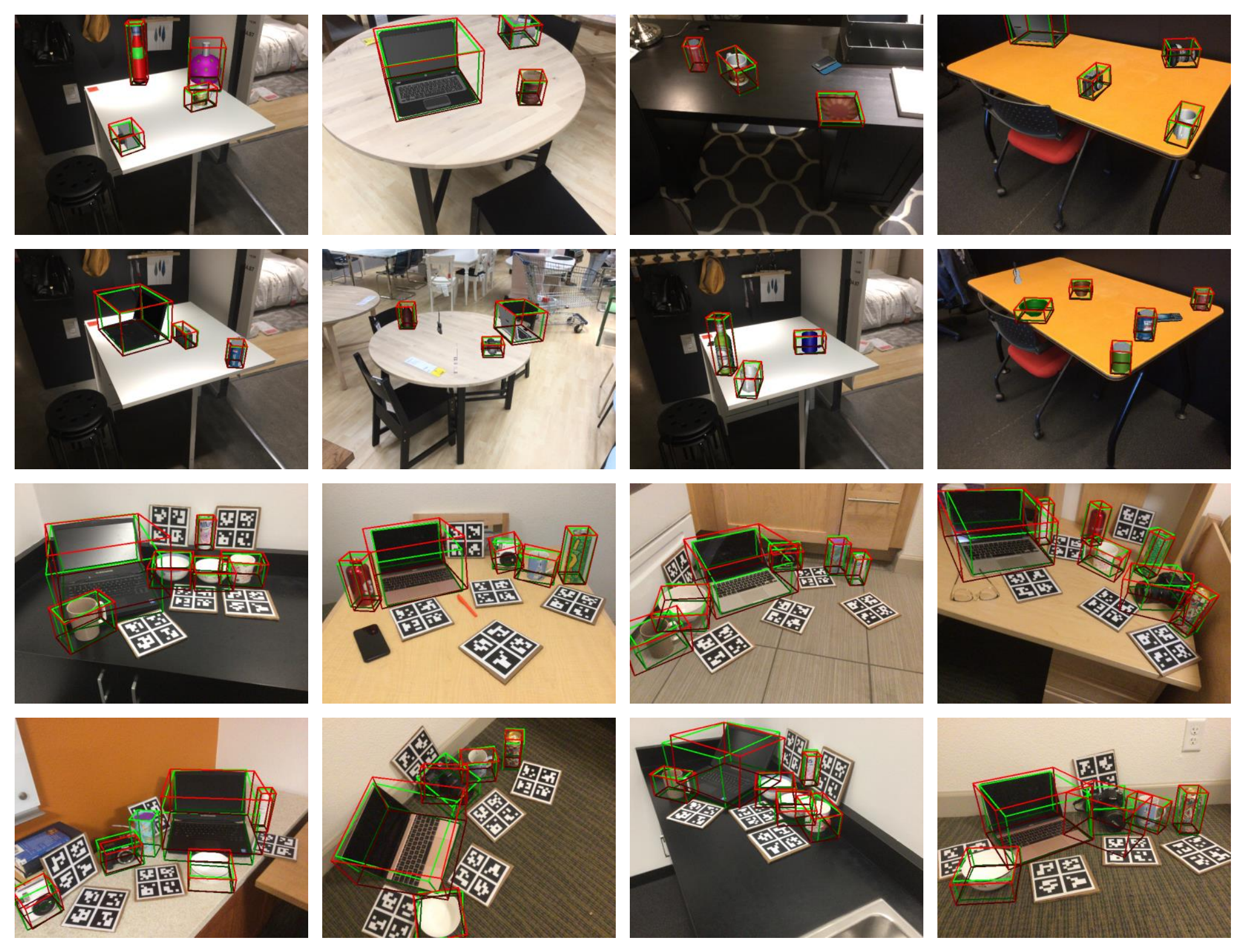}
    \caption{We show some qualitative results of our approach (red) and their ground truths (green) on CAMERA25 (top two rows) and REAL275 (bottom two rows).
    }
    \label{fig:sup_qualitative}
\end{figure}

\section{Runtime Analysis}
Given RGB-D images with resolution of $640 \times 480$ and mean object count of 4, our implementation approximately runs at 4 FPS on a desktop with an Intel Core i7-5960X CPU (3.0 GHz) and a NVIDIA GTX 1080Ti GPU.
Specifically, it takes an average time of 130 ms for instance segmentation, 100 ms for network inference, and 20 ms for pose alignment.

\section{Visualization of Shape Priors}
In Fig. \ref{fig:priors}, we visualize the different categorical shape priors used in our ablation studies.

\begin{figure}[t]
    \centering
    \includegraphics[width=\columnwidth]{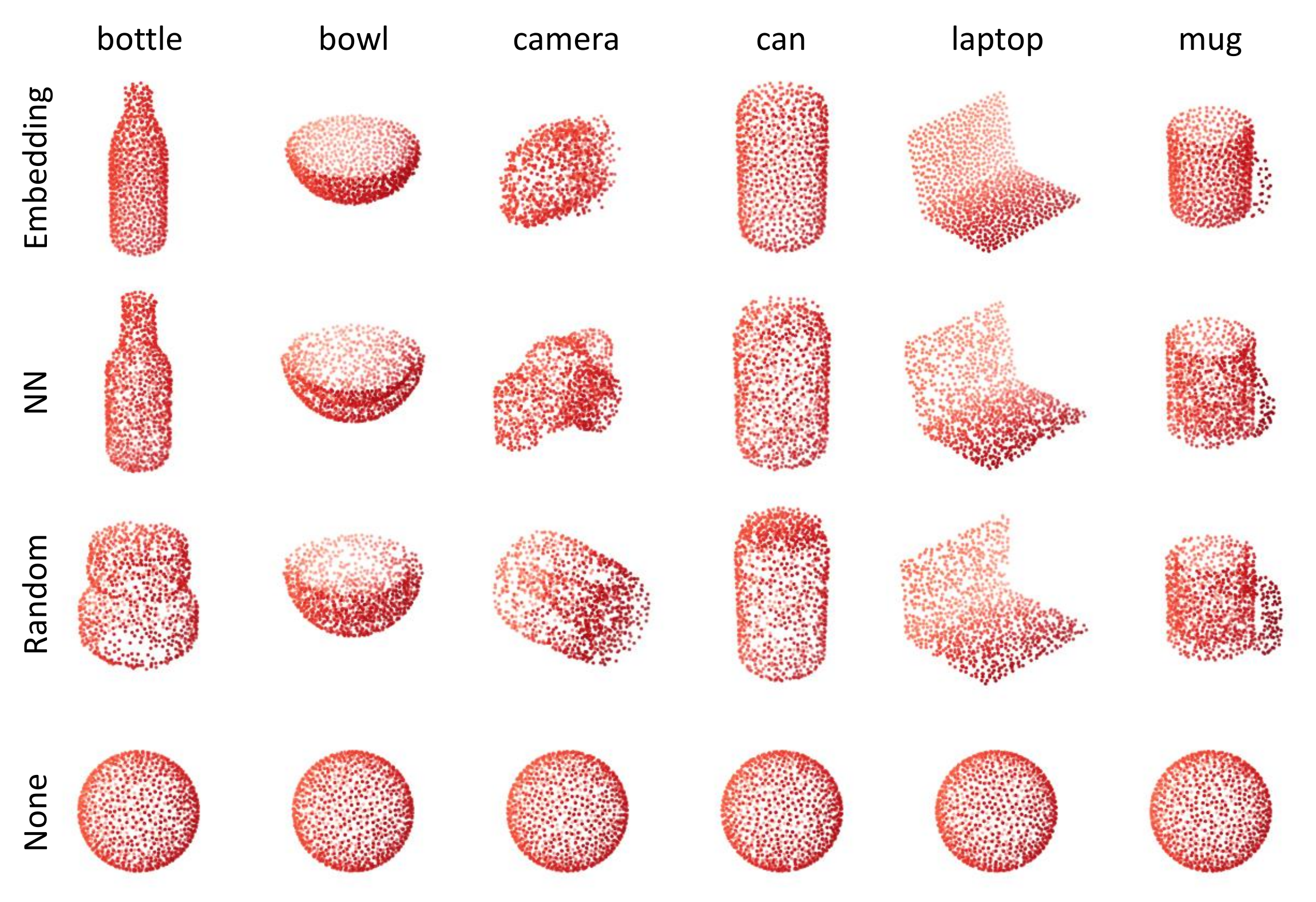}
    \caption{Categorical shape priors used in ablation studies.
    }
    \label{fig:priors}
\end{figure}

\section{Derivation of the Map Operator}

We first give the proposition from \cite{pitteri2019object}, and then derive the Map operator used in our work as a corollary. 
Given an object $M_c^i$, the proper symmetry group $\mathcal{S}(M_c^i)$ is defined as:
\begin{equation}
    \mathcal{S}(M_c^i) = \{ \mathbf{s} \in SO(3)\; \mid \; \forall \, \mathbf{p} \in SO(3), \; \mathcal{I}(M_c^i, \mathbf{p}) = \mathcal{I}(M_c^i, \mathbf{s} \cdot \mathbf{p}) \} ,
\end{equation}
where $\mathcal{I}(M_c^i, \mathbf{p})$ is the image of object $M_c^i$ under pose $\mathbf{p}$.
Intuitively, $\mathcal{S}(M_c^i)$ consists of rotations which preserve the appearance of a given object.

\begin{prop}\label{pro:symmetryGroup}
    Given a proper symmetry group $\mathcal{S}(M_c^i)$, $\forall \, R \in SO(3)$, the \textup{Map} operator is defined as:
    \begin{equation}
        \textup{Map} (R) = R \hat{S} , \; with \; \hat{S} = \underset{S \in \mathcal{S}(M_c^i)}{\arg\min} {\| R S - I_{3} \|_F} \; ,
    \end{equation}
    where $I_{3}$ is an $3\times3$ identity matrix. Then, $\textup{Map} (R_1) = \textup{Map} (R_2) \Longleftrightarrow \mathcal{I}(M_c^i, R_1) = \mathcal{I} (M_c^i, R_2) $ . 
\end{prop}
The proof is omitted for brevity, refer to \cite{pitteri2019object} for the details. The Map operator used in our work can then be derived directly from Proposition \ref{pro:symmetryGroup}.

\begin{corollary}
The Map operator for symmetrical objects around the y-axis is given by:
\begin{equation}
        \textup{Map} (R) = R \hat{S},~~~\forall \,R \in SO(3),
\end{equation}
where        
\begin{equation}
    \hat{S} =
        \setlength{\arraycolsep}{5pt}
        \begin{bmatrix}
        \cos \hat{\theta} & 0 & - \sin \hat{\theta} \\
        0 & 1 & 0 \\
        \sin \hat{\theta} & 0 & \cos \hat{\theta}
        \end{bmatrix}, \;
    \text{with} \; \hat{\theta} = \arctan 2 (R_{13} - R_{31}, R_{11} + R_{33}).
\end{equation}
\end{corollary}

\begin{proof}
Assuming the object $M_c^i$ is symmetrical around the y-axis of the object coordinate system, then $S$ has the following form:
\begin{equation}
    S =
        \setlength{\arraycolsep}{5pt}
        \begin{bmatrix}
        \cos \theta & 0 & - \sin \theta \\
        0 & 1 & 0 \\
        \sin \theta & 0 & \cos \theta
        \end{bmatrix} \; .
\end{equation}
The Froebenius norm can be rewritten as:
\begin{equation}
\begin{split}
     \| R S - I_{3} \|_F^2 & = 6 - 2 \text{Trace} (RS) \\
    & = 6 - 2 [ R_{11} \cos \theta + R_{13} \sin \theta + R_{22} + R_{33} \cos \theta - R_{31} \sin \theta ].
\end{split}
\end{equation}
We minimize the Froebenius norm over $\theta$ to solve for the Map:
\begin{equation}
\begin{split}
    \hat{\theta} & = \underset{\theta \in [0, 2 \pi )} {\arg\min} {\| R S - I_{3} \|_F} \\
    & = \underset{\theta \in [0, 2 \pi )} {\arg\max} {(R_{11} + R_{33}) \cos \theta + (R_{13} - R_{31}) \sin \theta} \\
    & = \arctan 2 (R_{13} - R_{31}, R_{11} + R_{33}) \; .
\end{split}
\end{equation}
Hence,
\begin{equation}
    \hat{S} =
        \setlength{\arraycolsep}{5pt}
        \begin{bmatrix}
        \cos \hat{\theta} & 0 & - \sin \hat{\theta} \\
        0 & 1 & 0 \\
        \sin \hat{\theta} & 0 & \cos \hat{\theta}
        \end{bmatrix}, \;
    \text{with} \; \hat{\theta} = \arctan 2 (R_{13} - R_{31}, R_{11} + R_{33}).
\end{equation}
\end{proof}

\end{document}